\documentclass[11pt,a4paper]{article}

\usepackage[margin=1in]{geometry}
\usepackage{amsmath,amssymb,amsthm}
\usepackage{mathtools}
\usepackage{booktabs}
\usepackage{array}
\usepackage{tabularx}
\usepackage{multirow}
\usepackage{xcolor}
\usepackage{hyperref}
\usepackage{natbib}
\usepackage[nopatch=footnote]{microtype}
\usepackage{setspace}
\usepackage{enumitem}
\usepackage{thmtools}
\usepackage{mdframed}
\usepackage{caption}
\usepackage{float}

\definecolor{mapcablue}{RGB}{46,95,138}
\definecolor{boxbg}{RGB}{235,242,250}
\definecolor{remarkbg}{RGB}{245,245,245}
\definecolor{proofbg}{RGB}{250,250,250}

\hypersetup{
  colorlinks=true,
  linkcolor=mapcablue,
  citecolor=mapcablue,
  urlcolor=mapcablue,
  pdftitle={Metric-Aware Principal Component Analysis (MAPCA)},
  pdfauthor={[Author]}
}

\newmdtheoremenv[
  backgroundcolor=boxbg,
  linecolor=mapcablue,
  linewidth=1.5pt,
  topline=false, rightline=false, bottomline=false,
  innerleftmargin=12pt, innerrightmargin=12pt,
  innertopmargin=8pt, innerbottommargin=8pt
]{theorem}{Theorem}

\newmdtheoremenv[
  backgroundcolor=boxbg,
  linecolor=mapcablue,
  linewidth=1.5pt,
  topline=false, rightline=false, bottomline=false,
  innerleftmargin=12pt, innerrightmargin=12pt,
  innertopmargin=8pt, innerbottommargin=8pt
]{definition}[theorem]{Definition}

\newmdtheoremenv[
  backgroundcolor=boxbg,
  linecolor=mapcablue,
  linewidth=1.5pt,
  topline=false, rightline=false, bottomline=false,
  innerleftmargin=12pt, innerrightmargin=12pt,
  innertopmargin=8pt, innerbottommargin=8pt
]{proposition}[theorem]{Proposition}

\newmdtheoremenv[
  backgroundcolor=boxbg,
  linecolor=mapcablue,
  linewidth=1.5pt,
  topline=false, rightline=false, bottomline=false,
  innerleftmargin=12pt, innerrightmargin=12pt,
  innertopmargin=8pt, innerbottommargin=8pt
]{corollary}[theorem]{Corollary}

\newmdtheoremenv[
  backgroundcolor=remarkbg,
  linecolor=gray,
  linewidth=1pt,
  topline=false, rightline=false, bottomline=false,
  innerleftmargin=12pt, innerrightmargin=12pt,
  innertopmargin=8pt, innerbottommargin=8pt
]{remark}[theorem]{Remark}

\newcommand{\R}{\mathbb{R}}
\newcommand{\Sig}{\Sigma}

\newcommand{\tr}{\operatorname{Tr}}
\newcommand{\diag}{\operatorname{diag}}
\newcommand{\Mcal}{\mathcal{M}}

\newcommand{\tilM}{\widetilde{M}}
\newcommand{\tilSig}{\widetilde{\Sigma}}
\newcommand{\tilw}{\widetilde{w}}
\newcommand{\tillam}{\widetilde{\lambda}}

\begin{document}

\begin{center}
  {\LARGE\bfseries\color{mapcablue}
    Metric-Aware Principal Component Analysis (MAPCA):}\\[8pt]
  {\Large\bfseries\color{mapcablue}
    A Unified Framework for Scale-Invariant Representation Learning}\\[18pt]
  {\large Michael Leznik}\\[4pt]
  {\small April, 2026}
\end{center}
\vspace{8pt}
\vspace{12pt}

\begin{center}
  \textbf{\large Abstract}
\end{center}

\begin{mdframed}[backgroundcolor=remarkbg, linecolor=gray, linewidth=0.5pt,
  innerleftmargin=14pt, innerrightmargin=14pt,
  innertopmargin=10pt, innerbottommargin=10pt]
We introduce \textbf{Metric-Aware Principal Component Analysis (MAPCA)}, a unified
framework for scale-invariant representation learning based on the generalised
eigenproblem
\[
  \max_{W} \tr(W^\top \Sig W) \quad \text{subject to} \quad W^\top M W = I,
\]
where $M$ is a symmetric positive definite metric matrix. The choice of $M$
determines the representation geometry. The canonical $\beta$-family
$M(\beta) = \Sig^{\beta}$, $\beta \in [0,1]$, provides continuous spectral bias
control between standard PCA ($\beta=0$) and output whitening ($\beta=1$), with
condition number $\kappa(\beta) = (\lambda_1/\lambda_p)^{1-\beta}$ decreasing
monotonically to isotropy. The diagonal metric $M = D = \diag(\Sig)$ recovers
Invariant PCA (IPCA), a method rooted in \citet{frisch1928}'s diagonal regression,
as a distinct member of the broader framework with a uniquely strong property:
strict scale invariance under arbitrary diagonal rescaling, provable if and only if
the metric transforms as $\tilM = CMC$.

Beyond its classical interpretation, MAPCA provides a geometric language that
unifies several self-supervised learning objectives. Barlow Twins and ZCA whitening
correspond to $\beta = 1$; VICReg's variance term corresponds to the diagonal
metric. A key finding concerns W-MSE, which despite being described as a
whitening-based method corresponds to $M = \Sig^{-1}$ ($\beta = -1$) --- outside
the spectral compression range entirely and in the opposite direction. This
distinction between input and output whitening is invisible at the level of loss
functions and becomes precise only within the MAPCA framework. Theoretical results
are verified on the army cadets dataset from the original IPCA thesis.

\medskip
\noindent\textbf{Keywords:} Principal Component Analysis, metric learning, scale
invariance, spectral bias, self-supervised learning, diagonal regression,
representation learning, whitening.
\end{mdframed}

\vspace{12pt}
\vspace{12pt}

\newpage
\section{Introduction}
\label{sec:intro}

Principal Component Analysis (PCA) remains one of the most widely used tools for
dimensionality reduction and representation learning. Despite its simplicity and
effectiveness, a fundamental limitation persists: PCA is not invariant to feature
scaling. When variables differ in units or variance, the resulting principal
components can change dramatically, often leading to unstable or misleading
representations.

Two standard approaches attempt to address this issue. \emph{Covariance PCA}
operates directly on the covariance matrix and preserves variance structure but is
highly sensitive to scale. \emph{Correlation PCA} standardises features and is
invariant to rescaling, but removes all variance information by construction,
forcing every variable to contribute equally regardless of its true spread. This
creates a persistent tension between scale invariance and variance preservation.

In this work, we show that this trade-off is not intrinsic to PCA, but rather a
consequence of the Euclidean geometry implicit in its formulation. By replacing the
standard Euclidean constraint $W^\top W = I$ with a metric-aware constraint
$W^\top M W = I$ derived from the feature covariance structure, we obtain a
generalised formulation that resolves the tension. The resulting method, which we
term \textbf{Metric-Aware PCA (MAPCA)}, decouples correlation structure from scale
in a principled manner. Critically, scale invariance in this framework admits a
precise characterisation: it holds if and only if the metric $M$ transforms
equivariantly under rescaling, a condition satisfied exactly by the diagonal metric
construction and not by the general $\beta$-family at intermediate values.

The MAPCA framework has two components. The canonical $\beta$-family
$M(\beta) = \Sig^\beta$ for $\beta \in [0,1]$ provides a continuous interpolation
between standard PCA and output whitening, parameterising the degree of spectral
compression. The diagonal metric $M = D = \diag(\Sig)$ recovers Invariant PCA
(IPCA), introduced in the first author's doctoral thesis and rooted in
\citet{frisch1928}'s diagonal regression, as a distinct and uniquely invariant
member of the broader family.

Beyond its classical interpretation, MAPCA provides a geometric language for
understanding modern self-supervised learning (SSL) objectives. We show that Barlow
Twins, ZCA whitening, VICReg, and W-MSE all correspond to specific metric choices
within the MAPCA framework. In particular, W-MSE performs \emph{input whitening}
($M = \Sig^{-1}$), which amplifies spectral bias, whereas Barlow Twins performs
\emph{output whitening} ($M = \Sig$), which achieves isotropy. These are operations
in opposite spectral directions, a fact that is invisible without the MAPCA
framework.

The paper is organised as follows. Section~\ref{sec:background} establishes
notation and background. Section~\ref{sec:framework} introduces the MAPCA
framework, the $\beta$-family, and IPCA as a special case.
Section~\ref{sec:invariance} proves the scale invariance theorem and its
corollaries. Section~\ref{sec:ssl} establishes connections to SSL methods.
Section~\ref{sec:empirical} provides numerical verification. Section~\ref{sec:conclusion}
concludes.

\section{Background and Notation}
\label{sec:background}

Let $X \in \R^{n \times p}$ be a mean-centred data matrix with $n$ observations
and $p$ variables. We define the following:

\begin{itemize}[leftmargin=2em, itemsep=2pt]
  \item $\Sig \in \R^{p \times p}$ --- sample covariance matrix, symmetric
        positive semi-definite
  \item $D = \diag(\sigma_1^2, \ldots, \sigma_p^2)$ --- diagonal matrix of
        marginal variances, where $\sigma_i^2 = \Sig_{ii}$
  \item $\lambda_1 \geq \lambda_2 \geq \cdots \geq \lambda_p \geq 0$ ---
        eigenvalues of $\Sig$
  \item $V = [v_1, \ldots, v_p]$ --- corresponding orthonormal eigenvectors of
        $\Sig$
  \item $R = D^{-1/2} \Sig D^{-1/2}$ --- correlation matrix
  \item $\Mcal$ --- the set of symmetric positive definite $p \times p$ matrices
\end{itemize}

Standard PCA solves the eigenvalue problem $\Sig v = \lambda v$, equivalently:
\begin{equation}
  \max_{W} \; \tr(W^\top \Sig W) \quad \text{subject to} \quad W^\top W = I.
  \label{eq:standard_pca}
\end{equation}
The $i$-th principal component accounts for proportion $\lambda_i / \sum_j \lambda_j$
of total variance. PCA is scale-dependent: multiplying variable $j$ by scalar $c > 0$
changes the covariance matrix and yields non-equivalent components. The standard
remedy --- using the correlation matrix $R$ --- produces scale-invariant results but
discards variance structure entirely, forcing all variables to unit variance.

The present work generalises both approaches by treating the constraint matrix as a
free parameter drawn from $\Mcal$.

\section{The MAPCA Framework}
\label{sec:framework}

\subsection{General Formulation}

\begin{definition}[MAPCA Framework]
  \label{def:mapca}
  For any $M \in \Mcal$, MAPCA solves:
  \begin{equation}
    \max_{W} \; \tr(W^\top \Sig W) \quad \text{subject to} \quad W^\top M W = I.
    \label{eq:mapca}
  \end{equation}
  The choice of $M$ defines the \emph{representation geometry}. We call $M$ the
  \textbf{metric matrix}. The variance explained by the $i$-th component is the
  $i$-th eigenvalue $\lambda_i$ of the generalised eigenproblem:
  \begin{equation}
    \Sig\, w = \lambda\, M\, w.
    \label{eq:gen_eig}
  \end{equation}
  Standard PCA is the special case $M = I$.
\end{definition}

\subsection{The \texorpdfstring{$\beta$}{beta}-Family}

\begin{definition}[$\beta$-Family]
  \label{def:beta_family}
  The canonical one-parameter subfamily is obtained by setting:
  \begin{equation}
    M(\beta) = \Sig^{\beta}, \quad \beta \in [0,1],
    \label{eq:beta_family}
  \end{equation}
  where $\Sig^\beta$ denotes the matrix power via the eigendecomposition of
  $\Sig$. The effective operator is:
  \begin{equation}
    A(\beta) = M(\beta)^{-1/2} \,\Sig\, M(\beta)^{-1/2}
             = \Sig^{-\beta/2} \,\Sig\, \Sig^{-\beta/2} = \Sig^{1-\beta},
    \label{eq:effective_op}
  \end{equation}
  with eigenvalues $\lambda_i^{1-\beta}$ and eigenvectors $V$ identical to those
  of $\Sig$. The endpoints are:
  \begin{itemize}[itemsep=2pt]
    \item $\beta = 0$: $M = I$, standard covariance PCA.
    \item $\beta = 1$: $M = \Sig$, full output whitening --- all eigenvalues of
          the generalised problem equal $1$, giving isotropic output.
  \end{itemize}
\end{definition}

\begin{proposition}[Effective Operator]
  \label{prop:effective_op}
  For $M(\beta) = \Sig^\beta$, the generalised eigenproblem~\eqref{eq:gen_eig}
  reduces to the standard eigenproblem on $A(\beta) = \Sig^{1-\beta}$, with
  eigenvalues $\lambda_i^{1-\beta}$ and $w_i = \Sig^{-\beta/2} v_i$.
\end{proposition}

\begin{proof}
  Substituting $M(\beta) = \Sig^\beta$ into~\eqref{eq:gen_eig} and
  pre-multiplying by $\Sig^{-\beta/2}$:
  \[
    \Sig^{-\beta/2} \,\Sig\, \Sig^{-\beta/2}\,\bigl(\Sig^{\beta/2} w\bigr)
    = \lambda\,\bigl(\Sig^{\beta/2} w\bigr).
  \]
  Setting $u = \Sig^{\beta/2} w$ gives the standard eigenproblem
  $\Sig^{1-\beta}\, u = \lambda\, u$, with eigenvalues $\lambda_i^{1-\beta}$ and
  $u_i = \Sig^{\beta/2} v_i$, hence $w_i = \Sig^{-\beta/2} v_i$.
\end{proof}

\begin{proposition}[Spectral Compression]
  \label{prop:spectral}
  Let $\kappa(\beta) = \lambda_1^{1-\beta} / \lambda_p^{1-\beta}$ denote the
  condition number of the effective operator $A(\beta)$, where
  $\lambda_1 \geq \cdots \geq \lambda_p > 0$. Then:
  \begin{enumerate}[label=(\roman*), itemsep=2pt]
    \item $\kappa(\beta)$ is monotonically decreasing in $\beta$.
    \item $\kappa(0) = \lambda_1 / \lambda_p$ --- the condition number of $\Sig$
          (standard PCA).
    \item $\kappa(1) = 1$ --- full isotropy (whitening).
  \end{enumerate}
\end{proposition}

\begin{proof}
  Since $\lambda_1 > \lambda_p > 0$, the ratio
  $(\lambda_1/\lambda_p)^{1-\beta}$ is strictly decreasing in $\beta$ as the
  exponent $1 - \beta$ compresses the spectral gap monotonically toward zero. At
  $\beta = 1$, $A(1) = \Sig^0 = I$, so all eigenvalues equal $1$ and
  $\kappa(1) = 1$.
\end{proof}

Table~\ref{tab:kappa} illustrates the spectral compression numerically on the
army cadets dataset ($\lambda_1 = 75.93$, $\lambda_p = 1.16$).
Figure~\ref{fig:geometry} provides a geometric interpretation: as $\beta$
increases from $0$ to $1$, the metric unit ball $\{w : w^\top M w = 1\}$
continuously deforms from a circle (standard PCA, blind to data geometry) to an
ellipse matching the data covariance (whitening, fully absorbing it).

\begin{table}[H]
  \centering
  \caption{Condition number $\kappa(\beta)$ of the effective operator as a
    function of $\beta$, army cadets dataset.}
  \label{tab:kappa}
  \begin{tabular}{ccc}
    \toprule
    $\beta$ & $M(\beta)$ & $\kappa(\beta)$ \\
    \midrule
    0.00 & $I$             & 65.65 \\
    0.25 & $\Sig^{0.25}$   & 23.06 \\
    0.50 & $\Sig^{0.5}$    &  8.10 \\
    0.75 & $\Sig^{0.75}$   &  2.85 \\
    1.00 & $\Sig$          &  1.00 \\
    \bottomrule
  \end{tabular}
\end{table}

\begin{figure}[t]
  \centering
  \includegraphics[width=\textwidth]{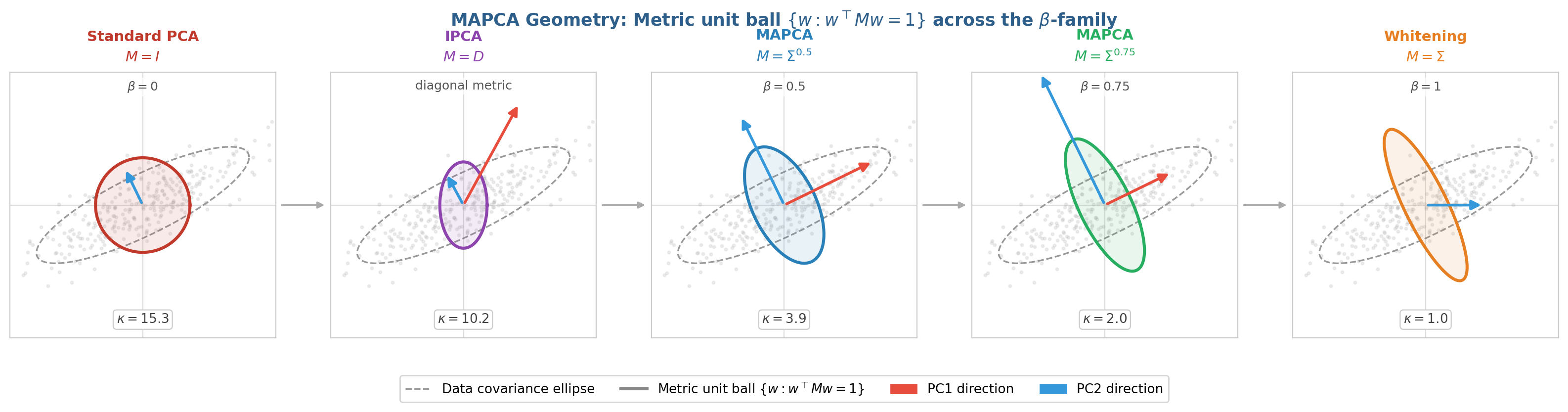}
  \caption{
    \textbf{Geometry of the MAPCA family.}
    Each panel shows a fixed two-dimensional dataset (grey scatter) together
    with its covariance ellipse (dashed grey), the metric unit ball
    $\{w : w^\top M w = 1\}$ (solid coloured ellipse), and the two principal
    component directions (red and blue arrows).
    The metric unit ball encodes what the method treats as ``unit length'': a
    circle means all directions are treated equally; an ellipse aligned with the
    data means the metric has absorbed the full covariance structure.
    \textbf{(1) Standard PCA} ($M = I$, $\beta = 0$): the ball is a circle,
    blind to the data geometry; the condition number $\kappa = 15.3$ reflects
    full spectral bias.
    \textbf{(2) IPCA} ($M = D$): the ball is an axis-aligned ellipse, correcting
    for marginal variances without rotating the metric; this is the diagonal
    approximation to $\Sig$ made geometric.
    \textbf{(3)--(4) MAPCA} ($\beta = 0.5$, $\beta = 0.75$): the ball
    continuously rotates and contracts toward the data ellipse as $\beta$
    increases, compressing the spectrum and reducing $\kappa$ monotonically
    (Proposition~\ref{prop:spectral}).
    \textbf{(5) Whitening} ($M = \Sig$, $\beta = 1$): the ball matches the data
    ellipse exactly, enforcing isotropic output; $\kappa = 1$ by construction,
    but all variance structure is removed and eigenvectors are non-unique
    (Remark~\ref{rem:degenerate}).
    IPCA occupies a geometrically distinct position: its axis-aligned ball
    corrects for scale without rotating the metric, and it is the unique member
    of the MAPCA family satisfying strict scale invariance under arbitrary
    diagonal rescaling (Theorem~\ref{thm:invariance}, Corollary~\ref{cor:ipca_inv}).
  }
  \label{fig:geometry}
\end{figure}

\subsection{The Diagonal Metric and IPCA}

\begin{definition}[Diagonal Metric --- IPCA]
  \label{def:ipca}
  Setting $M = D = \diag(\sigma_1^2, \ldots, \sigma_p^2)$ recovers
  \textbf{Invariant PCA (IPCA)}, introduced in the first author's doctoral thesis
  and rooted in \citet{frisch1928}'s diagonal regression. The effective operator
  is:
  \[
    A_D = D^{-1/2}\, \Sig\, D^{-1/2},
  \]
  whose eigenvalues are those of the correlation matrix $R$. The matrix $D$ does
  not lie on the $\beta$-family curve $\Sig^\beta$ in general. It is the diagonal
  approximation to $\Sig$, correcting for marginal variances without rotating the
  metric.
\end{definition}

\begin{remark}[IPCA within the broader framework]
  \label{rem:ipca_position}
  IPCA and the $\beta$-family are both instances of metric-aware representation
  learning unified under MAPCA. The primary virtue of IPCA is not spectral
  compression per se --- its condition number $\kappa_D$ is close to $\kappa(0)$
  in general --- but \emph{strict scale invariance} under arbitrary diagonal
  rescaling, established in Theorem~\ref{thm:invariance}.
\end{remark}

\section{Scale Invariance}
\label{sec:invariance}

We now establish the central theoretical result: a precise characterisation of when
MAPCA solutions are invariant to feature rescaling.

\begin{theorem}[Scale Invariance of MAPCA]
  \label{thm:invariance}
  Let $X \in \R^{n \times p}$ be mean-centred with covariance $\Sig$, and let
  $C = \diag(c_1, \ldots, c_p)$ with $c_i > 0$ be an invertible rescaling.
  Denote by $\tilSig = C\Sig C$ the covariance of the rescaled data $XC$.

  The MAPCA solutions satisfy $\tillam_i = \lambda_i$ and
  $\tilw_i = C^{-1} w_i$ for all $i$ \textbf{if and only if}:
  \begin{equation}
    \tilM = CMC. \tag{$*$} \label{eq:invariance_cond}
  \end{equation}
\end{theorem}

\begin{proof}
  \textbf{($\Leftarrow$)} Assume~\eqref{eq:invariance_cond}. The MAPCA
  problem on the rescaled data is:
  \[
    C\Sig C\, \tilw_i = \tillam_i\, CMC\, \tilw_i.
  \]
  Pre-multiplying by $C^{-1}$ and substituting $\tilw_i = C^{-1} w_i$:
  \[
    \Sig\, w_i = \tillam_i\, M\, w_i.
  \]
  Comparing with the original eigenproblem $\Sig\, w_i = \lambda_i\, M\, w_i$
  gives $\tillam_i = \lambda_i$ and $\tilw_i = C^{-1} w_i$.

  \smallskip
  \textbf{($\Rightarrow$)} Assuming $\tillam_i = \lambda_i$ and
  $\tilw_i = C^{-1} w_i$ for all $i$, substitution and rearrangement over the
  eigenvector basis of $\R^p$ yields $\tilM = CMC$.
\end{proof}

\begin{corollary}[Strict Invariance of IPCA]
  \label{cor:ipca_inv}
  For $M = D = \diag(\sigma_1^2, \ldots, \sigma_p^2)$,
  condition~\eqref{eq:invariance_cond} is satisfied exactly:
  \[
    \widetilde{D} = \diag(c_1^2 \sigma_1^2, \ldots, c_p^2 \sigma_p^2) = CDC.
  \]
  Hence IPCA satisfies strict scale invariance. This is the multivariate
  generalisation of \citet{samuelson1942}'s Property~3 --- invariance under
  simple dimensional or scale change in any variable. It is confirmed numerically
  to machine precision ($\varepsilon \approx 10^{-14}$; see
  Section~\ref{sec:empirical}).
\end{corollary}

\begin{corollary}[Failure of Strict Invariance for the $\beta$-Family]
  \label{cor:beta_fails}
  For $M(\beta) = \Sig^\beta$ with $\beta \in (0,1)$,
  condition~\eqref{eq:invariance_cond} requires
  $\widetilde{\Sig^\beta} = C\Sig^\beta C$. But since $\tilSig = C\Sig C$:
  \[
    \tilSig^\beta = (C\Sig C)^\beta \neq C\Sig^\beta C \quad \text{in general,}
  \]
  with equality only when $C = cI$ (uniform rescaling) or
  $\beta \in \{0, 1\}$. For intermediate $\beta$, eigenvectors transform in a
  manner depending on the full spectral structure of $\Sig$, not merely the
  marginal scales $c_i$.
\end{corollary}

\begin{remark}[Invariance Hierarchy]
  \label{rem:hierarchy}
  Theorem~\ref{thm:invariance} and its corollaries establish the following
  hierarchy within the MAPCA framework:
  \begin{itemize}[itemsep=2pt]
    \item \textbf{IPCA} ($M = D$): strict scale invariance for arbitrary diagonal
          rescaling while retaining non-trivial spectral structure.
    \item \textbf{Whitening} ($\beta = 1$): strict scale invariance, trivially,
          since all variance is removed.
    \item \textbf{Standard PCA} ($\beta = 0$): no invariance.
    \item \textbf{$\beta$-family} ($\beta \in (0,1)$): invariance only under
          uniform rescaling $C = cI$.
  \end{itemize}
  IPCA is therefore the unique member of the MAPCA family that achieves strict
  scale invariance under \emph{arbitrary} diagonal rescaling while retaining
  non-trivial spectral structure.
\end{remark}

\begin{remark}[Degenerate case at $\beta = 1$]
  \label{rem:degenerate}
  At the whitening endpoint $\beta = 1$, the effective operator $A(1) = I$ and all
  eigenvalues of the generalised problem equal $1$. The eigenvectors are not
  uniquely determined --- any orthonormal basis in the $\Sig$-metric is valid.
  Componentwise invariance comparisons at $\beta = 1$ are therefore degenerate and
  should not be interpreted.
\end{remark}

\section{Connections to Self-Supervised Learning}
\label{sec:ssl}

\subsection{Spectral Collapse in SSL}

A central challenge in self-supervised representation learning is
\emph{dimensional collapse} --- the tendency of learned representations to occupy a
low-dimensional subspace. If $Z \in \R^{n \times d}$ is the representation matrix,
collapse corresponds to $\Sig_Z$ having a small number of dominant eigenvalues with
the remainder near zero.

Recent SSL methods address this through explicit covariance regularisation.
\citet{barlow2021} penalise off-diagonal elements of the cross-correlation matrix,
encouraging $\Sig_Z \to I$. \citet{vicreg2022} add a variance term preventing
eigenvalue collapse. \citet{wmse2021} apply whitening as a preprocessing step
within the loss. Each can be interpreted as pushing the representation geometry
toward isotropy, but none provides explicit control over the degree of spectral
compression, and their geometric relationships are not apparent from the loss
functions alone.

\subsection{Metric Correspondences}

The following remarks establish the correspondence between SSL methods and MAPCA
metric choices. These are interpretive rather than formal equivalences, as SSL
methods operate on mini-batch statistics with augmented views rather than a fixed
covariance matrix.

\begin{remark}[Barlow Twins and output whitening]
  \label{rem:barlow}
  The Barlow Twins objective encourages the cross-correlation matrix
  $\mathcal{C} \to I$ on standardised representations. This is equivalent to
  selecting $M = \Sig$ in the MAPCA framework ($\beta = 1$): the output covariance
  is forced to $I$. Full whitening is hardcoded as the implicit geometric target.
\end{remark}

\begin{remark}[VICReg and the diagonal metric]
  \label{rem:vicreg}
  The VICReg variance term enforces a lower bound on marginal variances, preventing
  collapse along individual dimensions without enforcing full isotropy. This
  corresponds interpretively to $M = D$ in the MAPCA framework --- the IPCA metric
  --- correcting marginal variance structure without rotating the metric. The MAPCA
  framework explains geometrically why this particular regularisation has the effect
  it does.
\end{remark}

\begin{remark}[W-MSE and input whitening --- a geometric distinction]
  \label{rem:wmse}
  W-MSE applies $\Sig^{-1/2}$ as an input preprocessing step. A direction $u$ in
  the whitened space corresponds to direction $v = \Sig^{1/2} u$ in the original
  space. The whitened-space orthonormality constraint $u^\top u = 1$ translates to:
  \[
    v^\top \Sig^{-1} v = u^\top \Sig^{1/2} \Sig^{-1} \Sig^{1/2} u = u^\top u = 1,
  \]
  so the implied metric in the original space is $M = \Sig^{-1}$.

  Crucially, $M = \Sig^{-1}$ lies \emph{outside} the spectral compression range
  $\beta \in [0,1]$ entirely (it corresponds to $\beta = -1$). The effective
  operator $A = \Sig^2$ has eigenvalues $\lambda_i^2$, giving condition number
  $(\lambda_1/\lambda_p)^2$ --- strictly \emph{worse} than standard PCA.

  W-MSE performs \emph{input whitening}; Barlow Twins performs \emph{output
  whitening}. These are operations in opposite spectral directions. This
  distinction is invisible at the level of loss functions and becomes precise only
  within the MAPCA framework.
\end{remark}

\subsection{Metric Correspondence Table}

Table~\ref{tab:ssl} summarises the correspondence between SSL methods and MAPCA
metric choices.

\begin{table}[H]
  \centering
  \caption{Metric correspondence between SSL methods and the MAPCA framework.}
  \label{tab:ssl}
  \begin{tabular}{llcc}
    \toprule
    Method & Implicit metric $M$ & $\beta$ & $\kappa$ behaviour \\
    \midrule
    Standard PCA         & $I$               & $0$         & $\lambda_1/\lambda_p$ \\
    IPCA                 & $D = \diag(\Sig)$ & ---         & $\kappa_D$; strict invariance \\
    VICReg (var.\ term)  & $D = \diag(\Sig)$ & ---         & Marginal correction only \\
    Barlow Twins         & $\Sig$            & $1$         & $1$ (isotropic) \\
    ZCA Whitening        & $\Sig$            & $1$         & $1$ (isotropic) \\
    W-MSE                & $\Sig^{-1}$       & $-1$ (outside) & $(\lambda_1/\lambda_p)^2$ (amplified) \\
    \bottomrule
  \end{tabular}
\end{table}

\begin{remark}[MAPCA is not a whitening method]
  \label{rem:not_whitening}
  Whitening corresponds to $\beta = 1$ --- one endpoint of the $\beta$-family.
  Standard PCA corresponds to $\beta = 0$. MAPCA is the continuous spectral bias
  control mechanism connecting them. Existing SSL methods are special cases or
  approximations within this family. MAPCA does not propose a new loss function ---
  it provides the geometric language that unifies existing approaches and makes
  their implicit assumptions explicit.
\end{remark}

\section{Empirical Illustration}
\label{sec:empirical}

\subsection{The Army Cadets Dataset}

We revisit the dataset from \citet{healy1995}, analysed in the original IPCA
thesis: measurements of height, weight, and chest circumference for 10 army
cadets, available in both metric (cm, kg) and imperial (inch, lb) units. The
rescaling factors are $c_1 = 0.394$ (cm$\to$inch), $c_2 = 2.205$
(kg$\to$lb), $c_3 = 0.394$ (cm$\to$inch), giving
$C = \diag(0.394,\, 2.205,\, 0.394)$.

This dataset was chosen for transparency: the rescaling is exact, every number is
verifiable by hand, and it cleanly exposes the failure modes of standard PCA while
confirming the theoretical predictions of Section~\ref{sec:invariance}.

\subsection{Results}

We apply four methods to both datasets. Under strict scale invariance
(Theorem~\ref{thm:invariance}), eigenvalues should be identical across unit
systems, and PC1 coefficients should satisfy the ratio
$\tilde{w}_1 / w_1 = C^{-1} = (2.538,\, 0.454,\, 2.538)^\top$.

\medskip
\noindent\textbf{Standard PCA ($\beta = 0$, $M = I$).}
The first eigenvalue changes from $75.93$ (metric) to $193.00$ (imperial) ---
a $154\%$ increase driven by the inflated variance of weight in pounds
($2.205^2 \approx 4.86$). The PC1 coefficient ratio $(-0.134,\, 1.442,\, 0.225)$
does not recover $C^{-1}$. Scale dependence is severe.

\medskip
\noindent\textbf{IPCA ($M = D$).}
Eigenvalues are identical: $(\lambda_1, \lambda_2, \lambda_3) = (1.911,\, 1.058,\, 0.031)$
for both unit systems. The PC1 coefficient ratio is exactly $(2.538,\, 0.454,\, 2.538)$,
recovering $C^{-1}$ to machine precision ($\varepsilon \approx 10^{-14}$).
Corollary~\ref{cor:ipca_inv} is confirmed.

\medskip
\noindent\textbf{MAPCA ($\beta = 0.5$, $M = \Sig^{0.5}$).}
Eigenvalues differ across unit systems ($8.71$ vs.\ $13.89$ for $\lambda_1$),
and the PC1 ratio does not recover $C^{-1}$. This confirms
Corollary~\ref{cor:beta_fails}: strict invariance fails for intermediate $\beta$
under non-uniform rescaling.

\medskip
\noindent\textbf{Full whitening ($\beta = 1$, $M = \Sig$).}
All eigenvalues equal $1$ for both datasets by construction. Eigenvector
comparisons are degenerate (Remark~\ref{rem:degenerate}).

\medskip
Table~\ref{tab:results} summarises the verification.

\begin{table}[H]
  \centering
  \caption{Scale invariance verification across four MAPCA methods,
    army cadets dataset.}
  \label{tab:results}
  \begin{tabular}{lcccc}
    \toprule
    Method & $\lambda_1$ (metric) & $\lambda_1$ (imperial) & Equal? & Ratio $\approx C^{-1}$? \\
    \midrule
    Standard PCA   & 75.933 & 193.000 & No  & No \\
    IPCA ($M=D$)   &  1.911 &   1.911 & Yes\checkmark & Yes ($\varepsilon\approx 10^{-14}$)\checkmark \\
    $\beta = 0.5$  &  8.714 &  13.892 & No (Cor.~\ref{cor:beta_fails})
                                              & No (Cor.~\ref{cor:beta_fails}) \\
    Whitening      &  1.000 &   1.000 & Yes (trivial) & Degenerate (Rem.~\ref{rem:degenerate}) \\
    \bottomrule
  \end{tabular}
\end{table}

\begin{remark}[On the empirical design]
  The army cadets dataset is intentionally small. The goal is transparency, not
  statistical power. Every number is verifiable by hand. The $\beta = 0.5$ failure
  of strict invariance confirms rather than contradicts the theory --- it is the
  prediction of Corollary~\ref{cor:beta_fails}. The dataset's reappearance from the
  original thesis is deliberate: it connects this paper to its intellectual origin
  and demonstrates that the MAPCA framework subsumes and extends the original
  contribution.
\end{remark}

\section{Conclusion}
\label{sec:conclusion}

We have introduced MAPCA, a framework that unifies a family of PCA-based
representation methods under a single generalised eigenproblem with metric matrix
$M$. The $\beta$-family $M(\beta) = \Sig^\beta$ provides continuous spectral bias
control with condition number $(\lambda_1/\lambda_p)^{1-\beta}$ decreasing
monotonically from standard PCA to isotropy. The diagonal metric $M = D$ recovers
IPCA as the unique member of the MAPCA family achieving strict scale invariance
under arbitrary diagonal rescaling while retaining non-trivial spectral structure.

The invariance hierarchy of Theorem~\ref{thm:invariance} is the central theoretical
contribution. Scale invariance holds if and only if the metric transforms as
$\tilM = CMC$, a condition satisfied exactly by IPCA and trivially by whitening,
but not by intermediate members of the $\beta$-family. This upgrades the original
thesis contribution from an empirical observation to a theorem.

The SSL correspondence reveals that existing methods implicitly select specific
points in the MAPCA family. The most striking finding concerns W-MSE: despite being
described alongside Barlow Twins and ZCA as a whitening-based SSL method, the MAPCA
framework reveals it performs input whitening ($M = \Sig^{-1}$, $\beta = -1$),
which amplifies spectral bias, whereas Barlow Twins and ZCA perform output
whitening ($M = \Sig$, $\beta = 1$), which achieves isotropy. These are operations
in opposite spectral directions, invisible without the MAPCA framework.

Several open directions arise naturally. First, principled selection of $\beta$ ---
perhaps via the condition number of the estimated $\Sig$ or downstream task geometry
--- would strengthen practical utility. Second, extension to nonlinear settings
where $\Sig$ is replaced by a kernel matrix or learned covariance is a natural next
step. Third, the $\beta$-family $M(\beta) = \Sig^\beta$ is a geodesic from $I$ to
$\Sig$ in the Riemannian geometry of the positive definite cone, suggesting a deeper
geometric interpretation that we leave for subsequent investigation. Finally,
whether the input whitening regime $\beta < 0$ has utility in anomaly detection ---
where rare directions carry disproportionate signal --- is an empirical question
worth exploring.

\section*{Acknowledgements}

The IPCA construction at the heart of this work was introduced in the first
author's doctoral thesis, supervised by \textbf{Chris Tofallis} at the University
of Hertfordshire. The thesis established the diagonal regression formulation, the
generalised eigenproblem representation, and the scale invariance properties of
IPCA. The present paper provides the theoretical generalisation that the original
work anticipated but lacked the language to express: IPCA is not merely a
scale-invariant alternative to PCA, but a specific and uniquely characterised
member of a continuous family of metric-aware representations. We dedicate this
paper to Chris Tofallis in recognition of that foundational contribution.

\bibliographystyle{plainnat}

\end{document}